\documentclass{article}

\usepackage[preprint,nonatbib]{neurips_2020}

\usepackage[utf8]{inputenc} 
\usepackage[T1]{fontenc}    
\usepackage{hyperref}       
\usepackage{url}            
\usepackage{booktabs}       
\usepackage{amsfonts}       
\usepackage{nicefrac}       
\usepackage{microtype}      
\usepackage{mathtools}
\usepackage{amsmath}
\usepackage{amssymb}
\usepackage{amsthm}
\usepackage{tabularx}
\usepackage{color}
\usepackage{multirow}
\usepackage{algorithm}
\usepackage{algorithmic}
\usepackage{placeins}

\usepackage{ulem}

\newtheorem{theorem}{Theorem}
\newtheorem{innercustomthm}{Theorem}
\newenvironment{customthm}[1]
  {\renewcommand\theinnercustomthm{#1}\innercustomthm}
  {\endinnercustomthm}

\newtheorem{innerlemma}{Lemma}
\newenvironment{customlemma}[1]
  {\renewcommand\theinnerlemma{#1}\innerlemma}
  {\endinnerlemma}

\newtheorem{lemma}{Lemma}
\newtheorem{remark}{Remark}

\title{Analysis of Discriminator in RKHS Function Space for Kullback-Leibler Divergence Estimation}

\author{%
  Sandesh Ghimire, Prashnna K Gyawali, Linwei Wang  \\
  College of Computing and Information Sciences\\
  Rochester Institute of Technology\\
  Rochester, NY 14623, USA\\
  \texttt{\{sg9872, pkg2182, linwei.wang\}@rit.edu} \\
}

\begin{document}

\maketitle

\begin{abstract}
Several scalable sample-based methods to compute the Kullback–Leibler (KL) divergence between two distributions have been proposed and applied in large-scale machine learning models. While they have been found to be unstable, the theoretical root cause of the problem is not clear. In this paper, we study a generative adversarial network based approach that uses a neural network discriminator to estimate KL divergence. We argue that, in such case, high fluctuations in the estimates are a consequence of not controlling the complexity of the discriminator function space. We provide a theoretical underpinning and remedy for this problem by first constructing a discriminator in the Reproducing Kernel Hilbert Space (RKHS). This enables us to leverage sample complexity and mean embedding to theoretically relate the error probability bound of the KL estimates to the complexity of the discriminator in RKHS. Based on this theory, we then present a scalable way to control the complexity of the discriminator for a reliable estimation of KL divergence. We support both our proposed theory and method to control the complexity of the RKHS discriminator through controlled experiments. 
\end{abstract}

\section{Introduction}
\label{introduction}
Calculating {Kullback–Leibler} (KL) divergence from data samples is 
{an essential component in many} machine learning problems {that involve} Bayesian inference or the calculation of mutual information. In small data regime, this problem has been studied using variational technique and convex optimization \cite{nguyen2010estimating}.
In the presence of ever-increasing data,  
several neural network models have been proposed which require estimation of KL divergence 
such as total correlation variational autoencoder (TC-VAE) \cite{chen2018isolating}, adversarial variational Bayes (AVB) \cite{mescheder2018gan}, 
information maximizing GAN (InfoGAN) \cite{chen2016infogan}, 
and amortized MAP \cite{sonderby2016amortised}. These large scale models 
have imposed the following new requirements on estimating KL divergence: 
1. \underline{Scalability}: The estimation algorithm should be able to compute KL divergence from a large amount of data samples. 
2. \underline{Minibatch compatibility}: The algorithm should be compatible with minibatch-based optimization 
    and allow backpropagation (or other ways of optimizing the rest of the network) based on the {estimated} value of KL divergence.


These needs make {classic} methods 
{such as} \cite{nguyen2010estimating}
impractical, but  
were met by modern 
neural network based methods 
such as variational divergence minimization (VDM) \cite{nowozin2016f}, mutual information neural estimation (MINE) \cite{belghazi2018mutual}, and GAN-based KL estimation \cite{Mescheder2017ICML, sonderby2016amortised}.
A key attribute of these methods is that they are based on updating a neural-net discriminator function 
to estimate KL divergence 
from a subset of samples, which makes them scalable and minibatch compatible. 
We, however, noted that even in {simple} toy examples, these methods {tended} to be either unreliable (high fluctuation of estimates, 
as in GAN based approach by \cite{Mescheder2017ICML}), or unstable (discriminator yields infinity, as in MINE and VDM) (see Table.\ref{kl_table}). 
This behavior exacerbated when increasing the size of 
the discriminator. Similar {observations of} instability of VDM and MINE  have been reported in the literature \cite{Mescheder2017ICML,song2019understanding}.

In this paper, we 
{attempt to provide a theoretical underpinning for} the core problem of the large fluctuation in the 
{GAN based estimation of KL divergence}.
We approach this problem from the perspective of sample complexity, and propose that these fluctuations are a consequence of not controlling the complexity of the discriminator function. This direction has not been explored in existing works, and it faces the {open question} of how to properly 
measure the complexity of the large function space represented by neural networks. Note that naive approaches to bound complexity by the number of parameters would neither be guaranteed to yield tight bound, nor be easy to implement because it requires dynamically changing the size of the network during optimization. 

We introduce the following contributions 
to resolve this challenge. 
First, 
to be able to compute 
the complexity of the discriminator function space,
we propose 
{a novel construction of}
 the discriminator such that it lies in a smooth function space, the Reproducing Kernel Hilbert Space(RKHS). 
Leveraging 
sample complexity analysis and mean embedding of RKHS, 
we then 
bound the probability of the error of KL-divergence estimates 
in terms of the complexity of RKHS space. 
This further allows us to theoretically substantiate our main proposition 
that not controlling the complexity of the discriminator 
may lead to high fluctuation in estimation. 
Finally, we propose a scalable way to control the complexity of the discriminator based on the obtained error probability bound. In controlled experiments, we 
demonstrate that  
failing to control the complexity 
{of the discriminator function leads to fluctuation in  KL divergence estimates, and that the}
proposed method decreases such fluctuations.

\section{Related Works}
\vspace{-.2cm}

Nguyen et al \cite{nguyen2010estimating} used variational function to estimate KL divergence from samples of two distribution using convex risk minimization (CRM). They used the RKHS norm of the variational function as a way 
to both measure and penalize
the complexity of the variational function. However, their work required handling all data at once and solving a convex optimization problem 
that could not be scaled.
VDM reformulates the f-Divergence objective using Fenchel duality and uses a neural network to represent the variational function \cite{nowozin2016f}.
It is in concept close to 
\cite{nguyen2010estimating}, while the use of neural network and adversarial optimization made the estimation scalable. 
It however did not control the complexity of the neural-net function, resulting in unstable estimations.

One area of modern application of KL-divergence estimation is in computing mutual information which, as shown in MINE \cite{belghazi2018mutual}, is useful in applications such as stabilizing GANs or realizing the information bottleneck principle.  
MINE  also optimizes a lower bound, but tighter, to KL divergence (Donsker-Varadhan representation). Similar to VDM, MINE uses a neural network as the dual variational function: it is thus scalable, but without complexity control and unstable.

Another use of KL divergence is scalable variational inference (VI) as shown in AVB \cite{Mescheder2017ICML}.
VI requires KL divergence estimation between the posterior and the prior, which becomes nontrivial when 
an expressive posterior distribution 
is used and requires sample based scalable estimation. AVB solved it using GAN based adversarial formulation and a neural network discriminator. Similarly, \cite{sonderby2016amortised} used GAN based adversarial formulation to obtain KL divergence in amortized inference.


To disentangle latent representations in VAE, \cite{chen2018isolating} proposed 
TC-VAE which penalized the KL divergence between marginal latent distribution and the product of marginals in each dimension. This KL divergence was computed by minibatch based sampling strategy that gives a biased estimate. 
None of the existing works considered 
the theoretical underpinning 
of unreliable KL-divergence estimates, 
or mitigating the problem by controlling the complexity of the discriminator function.

\section{Preliminaries}
\textbf{Reproducing Kernel Hilbert Space}:
Let $\mathcal{H}$ be a Hilbert space of functions $f:\mathcal{X}\to {\rm I\!R}$ defined on non-empty space $\mathcal{X}$. It is a Reproducing Kernel Hilbert Space (RKHS) if the evaluation functional, $\delta_x :\mathcal{H} \to {\rm I\!R}$, $\delta_x :f \mapsto f(x)$, is linear continuous $\forall x \in \mathcal{X}$. Every RKHS, $\mathcal{H}_K$, is associated with a unique positive definite kernel, $K: \mathcal{X}\times \mathcal{X}\to {\rm I\!R}$, called reproducing kernel \cite{berlinet2011reproducing}, such that it satisfies:
\begin{enumerate}
    \item $\forall x \in \mathcal{X}, K(.,x) \in \mathcal{H}_K$ \hspace{.3cm} (\textit{Membership property})
    \item $\forall x \in \mathcal{X}, \forall f \in \mathcal{H}_K,\hspace{.1cm} \langle f,K(.,x) \rangle_{\mathcal{H}_K}=f(x)$ \hspace{.2cm} (\textit{Reproducing property})
\end{enumerate}
RKHS is often studied using a specific integral operator. Let $\mathcal{L}_2(d\rho)$ be a space of functions $f: \mathcal{X} \to {\rm I\!R}$ that are square integrable with respect to a Borel probability measure $d\rho$ on $\mathcal{X}$, we define an integral operator  $L_K:\mathcal{L}_2(d\rho) \to \mathcal{L}_2(d\rho)$ \cite{bach2017equivalence, cucker2002mathematical}:
$   (L_K f)(x)=\int_{\mathcal{X}}f(y)K(x,y)d\rho(y) $
This operator will be important in constructing a function in RKHS and in computing sample complexity.

\textbf{Mean Embedding in RKHS}:
Let $f:\mathcal{X}\to {\rm I\!R}$ be a function in {RKHS}, $\mathcal{H}_K$, and $p$ be a Borel probability measures on $\mathcal{X}$. If $E_{x\sim p}\sqrt{K(x,x)}<\infty$, then we have $\mu_p \in \mathcal{H}_K$ called the mean embedding of the distribution $p$ and defined as \cite{sriperumbudur2010hilbert,berlinet2011reproducing,gretton2012kernel}: 
$    E_{x\sim p}f=\langle f,\mu_p \rangle_{\mathcal{H}_K}$
The condition for the existence of mean embedding is readily satisfied since we assume 
$\underset{x,t}{sup}\hspace{0.1cm}{K(x,t) < \infty}$. 

\section{Problem Formulation and Contribution}

\textbf{GAN-based Estimation of KL Divergence:}
Let $p(x)$ and $q(x)$ be two probability density functions 
in space $\mathcal{X}$ and we want to estimate their KL divergence using finite samples from each distribution  
in a scalable and minibatch compatible manner. 
As shown in \cite{Mescheder2017ICML, sonderby2016amortised}, this can be achieved by using a discriminator function. First, a discriminator $f:\mathcal{X}\to {\rm I\!R}$ is trained with the objective:
\begin{align}
\label{gan_kl}
    f^*=\underset{f}{argmax}[{E_{p(x)}\log \sigma(f(x))+E_{q(x)}\log (1-\sigma(f(x)))}] 
\end{align}
where $\sigma$ is the Sigmoid function given by $\sigma(x)=\frac{e^x}{1+e^x}$.
Then it can be shown \cite{Mescheder2017ICML, sonderby2016amortised} that the KL divergence $KL(p(x)||q(x))$ is given by:
\begin{align}
\label{kl_expectation}
 KL(p(x)||q(x))=E_{p(x)}[f^*(x)]   
\end{align}

\textbf{Sources of Error:}
Typically, a neural network is used as the discriminator. This implies that we are considering the space of functions represented by the neural network of given architecture as the hypothesis space, over which the maximization occurs in eq.(\ref{gan_kl}). We thus must rewrite eq.(\ref{gan_kl}) as 
\begin{align}
\label{finite_opt}
    f^*_h=\underset{f \in h}{argmax}[{E_{p(x)}\log \sigma(f(x))+E_{q(x)}\log (1-\sigma(f(x)))}] 
\end{align}
where $h$ is the discriminator function space. Furthermore, we have only a finite number of samples, say $m$, from the distribution $p$ and $q$. Then, under finite sample, the optimum discriminator is
\begin{align}
\label{kl_h}
    f^m_h=\underset{f \in h}{argmax}\Big[{\frac{1}{m}\sum_{x_i \sim p(x_i)}\log \sigma(f(x_i))+\frac{1}{m}\sum_{x_j \sim q(x_j)}\log (1-\sigma(f(x_j)))}\Big] 
\end{align}
Similarly, we write KL estimate obtained from, respectively, infinite and finite samples as:
\begin{align}
\label{kl_finite}
 KL(f)=E_{p(x)}[f(x)],   \hspace{0.4cm}  KL_m(f)=\frac{1}{m}\sum_{x_i \sim p(x_i)}f(x)
\end{align}
With these definitions, we can now write the error in estimation as:
\begin{align}
\label{total_error}
KL_m(f^m_h)-KL(f^*)=\underbrace{KL_m(f^m_h)-KL(f^m_h)}_{\text{Deviation-from-mean error}}+\underbrace{KL(f^m_h)-KL(f^*_h)}_{\text{Discriminator induced error}}+\underbrace{KL(f^*_h)-KL(f^*)}_{Bias}
\end{align}
This equation decomposes total estimation error into three terms: 1) deviation from the mean error, 2) 
error in KL estimate 
by 
the discriminator 
due to using finite samples in optimization eq.(\ref{kl_h}), and 3) bias when the considered function space 
does not contain optimal function.
We leave quantification of second and third term as future work. Here, we concentrate on quantifying the probability of deviation-from-mean error which is directly related to observed variance of the KL estimate. 

\textbf{Overview of Contributions:}
Note that the deviation is the difference between a random variable and its mean. Based on this observation, we can bound the probability of this error using concentration inequality and the complexity of function space of $f^m_h$. {This requires overcoming the open challenge} of measuring the complexity of neural networks function space. 
To this end, we propose to construct a function out of neural network such that it lies on RKHS. This is our first contribution (Section \ref{construct_rkhs}). Then, we proceed to bound the probability of deviation-from-mean error through the covering number of the RKHS space. Lemma 1 and Theorem 2 are our contribution (Section \ref{bounding_error}). Then we provide insight into how the optimization of eq.(\ref{kl_h}) might affect discriminator function space $\mathcal{H}_K$. Using ideas from mean embedding, we prove Lemma 3 and Theorem 3 and provide a geometric insight (Section \ref{mean_embedding}). 
This allows us to present a complete story of how the optimization setup might encourage increase in the complexity of $\mathcal{H}_K$ and how to control it (Section \ref{fitting_pieces}). 
\section{Constructing $f$ in RKHS}
\label{construct_rkhs}

To construct a function in RKHS, we use an operator $T$ related to integral operator $L_K$ by $L_K=TT^*$ \cite{bach2017equivalence}. The following theorem due to \cite{bach2017breaking} paves a way for us to construct a neural function in RKHS.


\begin{theorem}{[\cite{bach2017breaking} Appendix A]}
\label{rkhs_construct}
A function $f \in \mathcal{L}_2(d\rho)$ is in Reproducing Kernel Hilbert Space, say $\mathcal{H}_{K}$ if and only if it can be expressed as
\begin{align}
    \forall x \in \mathcal{X}, f(x)=\int_{\mathcal{W}}g(w)\psi(x,w)d\tau(w),
\end{align}{}
for a certain function $g: \mathcal{W}\to R$ such that $||g||^2_{\mathcal{L}_2(d\tau)} < \infty$. The RKHS norm of $f$ satisfies 
$
 ||f||^2_{\mathcal{H}_{K}} \leq  ||g||^2_{\mathcal{L}_2(d\tau)}
$
and the kernel $K$ is given by
\begin{align}
    \label{kernel}
    K(x,t)=\int_{\mathcal{W}}\psi(x,w)\psi(t,w)d\tau(w)
\end{align}

\end{theorem}
Theorem \ref{rkhs_construct} gives us a condition when a square integrable function is guaranteed to lie in RKHS.
We simply choose $g(w)=\textbf{1}$, a constant unit function over the domain $\mathcal{W}$. This means that we can convert a square integrable neural network function $f: \mathcal{X} \to {\rm I\!R}$ into a function in RKHS, if we make some weights in the neural network stochastic 
and average over them. Here, we make the last layer of the neural network to be drawn from Gaussian distribution, whose parameters are learnt during training. 
More precisely, we consider $\psi(x,w)=\phi_{\theta}(x)^Tw$, where $\phi_{\theta}(x)$ denotes neural network transformation until the last layer, and $w$ is the last linear layer sampled from Gaussian distribution. While in principle any layer could be made stochastic, we chose this architecture to reduce the computational cost of sampling. The kernel $K$, as defined in eq.(\ref{kernel}), can be obtained as:
\begin{align}
     K_{\theta}(x^*,t^*)&=\int_{\mathcal{W}}\phi_{\theta}(x^*)^Tww^T\phi_{\theta}(t^*) d\tau(w)
     =\phi_{\theta}(x^*)^T(\bar{w}\bar{w}^T+\Sigma)\phi_{\theta}(t^*) 
\end{align}
where $\bar{w}$ and $\Sigma$ denote the mean and covariance of $w$. 
We sometimes denote the kernel $K$ by $K_{\theta}$ to emphasize that it is a function of neural network parameters, $\theta$.

With this construction, our discriminator function $f$ lies in RKHS denoted by $\mathcal{H}_K$. With $g(w)=\textbf{1}$, it is easy to verify that $||g||^2_{L_2(d\tau)} =1$ since $w$ is sampled from a normal distribution. 
The inequality in Theorem \ref{rkhs_construct} gives us $||f||^2_{\mathcal{H}_{K}} \leq 1$. It is interesting that the RKHS norm of function $f$ is upper-bounded by 1. Traditionally, kernel $K$ remains fixed and the norm of the function $f$ determines the complexity of the function space. For example, \cite{nguyen2010estimating} penalized the $||f||_{\mathcal{H}_K}$ as a way to control the function space while estimating KL divergence. 
{In our RKHS formulation of neural networks,} {the nature of the problem has changed:} $||f||_{\mathcal{H}_K}$ cannot increase beyond 1, but
the RKHS itself changes during training since it is determined by the kernel that depends on neural parameters  $\theta$. Therefore, 
{the challenge becomes}
teasing out how neural {parameters $\theta$} 
affects the complexity of the discriminator function space and how that affects the deviation-from-mean error in  eq.(\ref{total_error}).

\section{Bounding the Error Probability of KL Estimates}
\label{bounding_error}
In this section, we first bound the probability of deviation-from-mean error 
in terms of the covering number 
in Lemma \ref{sample_complexity}. 
We then use an estimate of the covering number of RKHS due to \cite{cucker2002mathematical} to 
relate the bound 
to kernel $K_{\theta}$ in Theorem \ref{complexity}, 
identifying the role of neural networks in this error bound. 


\begin{lemma}
\label{sample_complexity}
Let $f^m_{\mathcal{H}_K}$ be the optimal discriminator function in a RKHS $\mathcal{H}_{K}$ which is M-bounded. Let ${KL}_m(f^m_{\mathcal{H}_K})=\frac{1}{m}\sum_i f^m_{\mathcal{H}_K}(x_i)$ and $KL(f^m_{\mathcal{H}_K}) = E_{ p(x)}[f^m_{\mathcal{H}_K}(x)]$ be the estimate of KL divergence from m samples and that by using true distribution $p(x)$ respectively.
Then the probability of error at some accuracy level, $\epsilon$ is lower-bounded as:
\begin{align}
\nonumber
    \text{Prob.}(&|{KL}_m(f^m_{\mathcal{H}_K})-{KL}(f^m_{\mathcal{H}_K})|\leq \epsilon) 
    \geq 1-2\mathcal{N}(\mathcal{H}_K, \frac{\epsilon}{4\sqrt{S_K}})\exp(-\frac{m\epsilon^2}{4M^2})
\end{align}
where $\mathcal{N}(\mathcal{H}_K,\eta)$ denotes the covering number of a RKHS space $\mathcal{H}_K$ with disks of radius $\eta$, and $S_K=\underset{x,t}{sup} $\hspace{0.1cm} ${K(x,t)}$ which we refer as kernel complexity
\end{lemma}
\begin{proof}[Proof Sketch]
We cover RKHS with discs of radius $\eta=\frac{\epsilon}{4\sqrt{S_K}}$. Within this radius, the deviation does not change too much. So, we can bound deviation probability at the center of disc and apply union bound over all the discs. See supplementary materials for the full proof.
\end{proof}
\vspace{-0.2cm}
Lemma \ref{sample_complexity} bounds the probability of error in terms of the covering number of the RKHS space. Note that the radius of the disc is inversely related to $S_K$ which 
indicates how complex the RKHS space defined by the kernel $K_{\theta}$ is. 
Here $K_{\theta}$ depends on the neural network parameters $\theta$. Therefore, we 
denote $S_K$ as a function of $\theta$ as $S_K(\theta)$ and 
term it kernel complexity. Next, we use Lemma 2 due to \cite{cucker2002mathematical} 
to obtain an error bound in estimating KL divergence with finite samples in Theorem \ref{complexity}.

\begin{lemma}[\cite{cucker2002mathematical}]
\label{covering number}
Let $K: \mathcal{X}\times \mathcal{X}\to {\rm I\!R} $ is a $\mathcal{C}^\infty$ Mercer kernel and the inclusion $I_K:\mathcal{H}_K\xhookrightarrow{}\mathcal{C}(\mathcal{X})$ is the compact embedding defined by $K$ to the Banach space $\mathcal{C}(\mathcal{X})$ . Let $B_R$ be the ball of radius $R$ in RKHS $\mathcal{H}_{K}$. Then $\forall \eta>0, R >0, h>n $, we have
\begin{align}
    \ln \mathcal{N}(I_K(B_R), \eta) \leq \left( \frac{RC_h}{\eta} \right)^{\frac{2n}{h}}
\end{align}
where $\mathcal{N}$ gives the covering number of the space $I_K(B_R)$ with discs of radius $\eta$, and $n$ represents the dimension of inputs space $\mathcal{X}$. $C_h$ is given by 
$
C_h=C_s\sqrt{||L_s||}
$
where $L_s$ is a linear embedding from square integrable space $\mathcal{L}_2(d\rho)$ to the Sobolev space $H^{h/2}$ and $C_s$ is a constant.
\end{lemma}
To prove Lemma \ref{covering number}, the RKHS space is embedded in the Sobolev Space $H^{h/2}$ using $L_s$ and then the covering number of the Sobolev space is used. Thus the norm of $L_s$ and the degree of Sobolev space, $h/2$, appears in the covering number of a ball in $\mathcal{H}_K$. 
In Theorem \ref{complexity}, we use this Lemma to bound the {estimation error of KL divergence}. 

\begin{theorem}
\label{complexity}
Let ${KL}(f^m_{\mathcal{H}_K})$ and ${KL}_m(f^m_{\mathcal{H}_K})$ be the estimates of KL divergence obtained by using true distribution $p(x)$ and $m$ samples respectively as described in Lemma \ref{sample_complexity}, then the probability of error in the estimation at the error level $\epsilon$ is given by:
\begin{align*}
   \text{Prob.}(&|{KL}_m(f^m_{\mathcal{H}_K})-{KL}(f^m_{\mathcal{H}_K})|\leq \epsilon) \geq 1-2\exp\Bigg[\left( \frac{4RC_s\sqrt{S_K(\theta)||L_s||}}{\epsilon} \right)^{\frac{2n}{h}}-\frac{m\epsilon^2}{4M^2}\Bigg]
\end{align*}
\end{theorem}
\begin{proof}
Lemma \ref{covering number} gives the covering number of a ball of radius $R$ in a RKHS space. If we consider the hypothesis space to be a ball of radius $R$ in Lemma \ref{sample_complexity} , we can apply Lemma \ref{covering number} in it. We fix the radius of discs to $\eta=\frac{\epsilon}{4\sqrt{S_K}}$ in Lemma \ref{sample_complexity} and 
substitute $C_h=C_s\sqrt{||L_s||}$ to obtain the required result. 
\end{proof}
\vspace{-0.1cm}
Theorem 2 shows that the error increases exponentially with the radius of the RKHS space, complexity of the kernel $S_K(\theta)$, and the norm of Sobolev space embedding $||L_s||$. Since we have $||f||_{\mathcal{H}_K} \leq 1$, we can consider our hypothesis space to be a ball of radius 1. To bound $||L_s||$, we need to compute higher order derivatives of $K(x,t)$, which we leave as future work. This allows us to focus on kernel complexity $S_K(\theta)$, which is exponentially related to the probability of deviation-from-mean error. 

Note that to bound the deviation-from-mean error, we used union bound and therefore, the bound does not explicitly depend on the function $f^m_{\mathcal{H}_K}$, but only depends on the complexity $S_K$ of the function space $\mathcal{H}_K$. However, the optimization of discriminator (eq.(\ref{kl_h})) also impacts the complexity $S_K(\theta)$. 
To 
understand this effect, 
in the next section, we present an upper bound on the objective in eq.(\ref{kl_h}), and give some geometric intuition connecting the optimization objective with the kernel complexity $S_K(\theta)$. Using this intuition, we further argue that the optimization of eq.(\ref{kl_h}) may encourage increment in the complexity, $S_K(\theta)$, thereby increasing the probability of deviation from the mean.


\section{Mean Embedding Upper Bound}
\label{mean_embedding}
In addition to deriving complexity bound, 
another advantage of using RKHS is that it allows us to use mean embedding representation. This helps us derive some geometrical insights into the maximization objective in eq.(\ref{kl_h}), 
on which we give an upper bound
in Theorem \ref{mebub}.  

\begin{theorem}
\label{mebub}
Let $f \in \mathcal{H}_{K_{\theta}}$ be a function in RKHS $\mathcal{H}_{K_{\theta}}$. Then we have the following upper bound on the objective of KL divergence estimation:
\begin{align}{}
{\frac{1}{m}\sum_{x_i \sim p(x)}\log \sigma(f(x_i))+\frac{1}{m}\sum_{x_j \sim q(x)}\log (1-\sigma(f(x_j)))}
\leq\log\sigma[\langle \mu^m_p-\mu^m_q, f \rangle_{\mathcal{H}_K} ]
\end{align}
and the KL divergence is given by
$\text{KL}=\langle \mu^m_p, f \rangle$
where $\mu^m_p$ and $\mu^m_q$ represent mean embedding of $m$ samples from distributions $x_i \sim p(x)$ and $x_j \sim q(x)$ with respect to $\mathcal{H}_K$.
\end{theorem}

The following Lemma is useful to prove this theorem.
\begin{lemma}
\label{upper_bound}
$
\nonumber
{E_{p(x)}\log \sigma(f(x))+E_{q(y)}\log (1-\sigma(f(y)))}
\leq \log \sigma[E_{p(x)}(f(x))-E_{q(y)}(f(y))]
$
\end{lemma}
\begin{proof}[Proof Sketch]
It is proved by using $1-\sigma(f)=\sigma(-f)$ and using Jensen's inequality since $\log\sigma$ is a concave function. See supplementary material for full proof.
\end{proof}

\begin{proof}[Proof of Theorem \ref{mebub}]
If $f$ lies in the RKHS, then there exists some $\mu^m_p$ and $\mu^m_q$ such that 
\begin{align}
    \label{finite_embed}
    \frac{1}{m}\sum_{x_i \sim p(x)} f(x_i)= \langle \mu^m_p, f \rangle_{\mathcal{H}_K} , \hspace{0.2cm}
   \frac{1}{m}\sum_{x_j \sim q(x)} f(x)= \langle \mu^m_q, f \rangle_{\mathcal{H}_K} 
\end{align}
Applying eq.(\ref{finite_embed}) to the Lemma \ref{upper_bound} for finite samples, we obtain required result.
\end{proof}
\textbf{Geometric Intuition:} Theorem \ref{mebub} tells us that the upper bound (MEBUB) to the objective is $\log \sigma$ of the inner product between $\mu^m_p-\mu^m_q$ and $f$. The inner product and KL divergence estimates have been depicted geometrically in Fig.~\ref{geometric}. When the objective is maximized, MEBUB may also increase which leads to an increase in the inner product since $\log \sigma$ is monotonic. When this happens, nothing {prevents} the midpoint of the mean embeddings, \textit{i.e.}, $\frac{\mu^m_p+\mu^m_q}{2}$, from going away from the origin in Fig.~\ref{geometric}. In the next section, we show how 
this affects kernel complexity $S_K$.
\begin{figure}[tbp]
\begin{center}
\centerline{\includegraphics[width=0.6\linewidth]{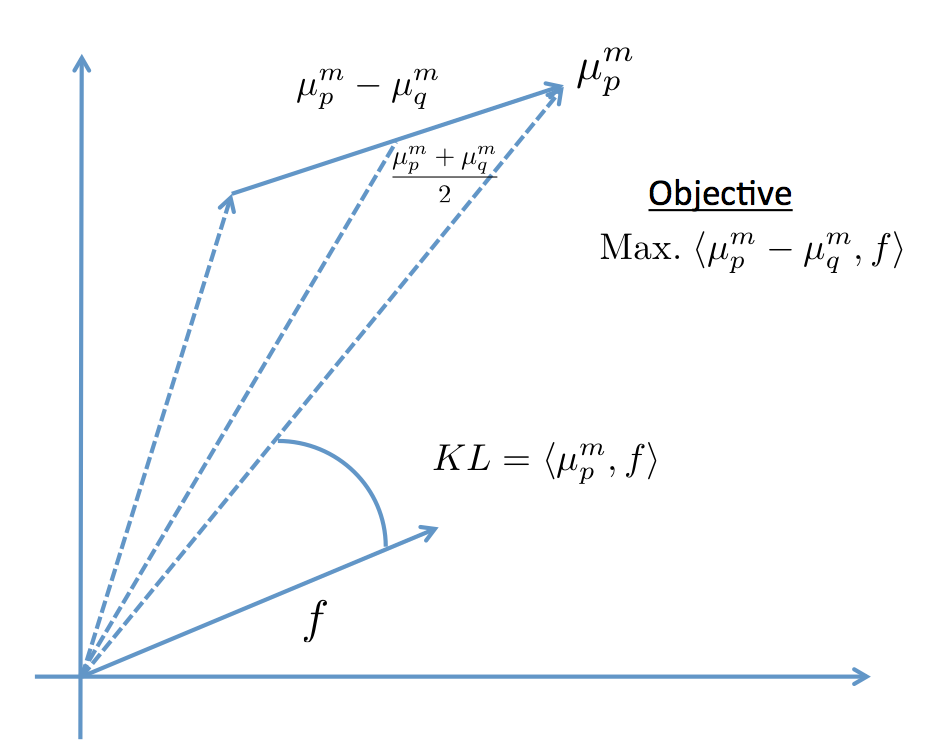}}
\caption{Geometrically representing mean embeddings of two distributions, and their relation to the maximization objective and KL divergence. 
\vspace{-.2cm}}
\label{geometric}
\end{center}
\end{figure}

\section{Fitting Pieces and Complexity Control}
\label{fitting_pieces}
Theorem \ref{complexity} shows that the error bound {of the KL estimate} is exponentially controlled by the kernel complexity $S_K(\theta)=\underset{x,t}{sup}\hspace{0.1cm} K_{\theta}(x,t)$. Since the mean of a vector is upper bounded by supremum,
\begin{align}
||\mu^m_p+\mu^m_q||_{\mathcal{H}_K}&=\sqrt{\frac{1}{m^2}\sum_{i,j} K_{\theta}(y_i,y_j)+2K_{\theta}(y_i,z_j)+K_{\theta}(z_i,z_j)}\\
\label{sup_avg}
&\leq 2\sqrt{\underset{x\in\{Y,Z\},t\in\{Y,Z\}}{sup} K_{\theta}(x,t)} = 2\sqrt{S_K(\theta)}
\end{align}
As the training progresses {in maximizing the objective in eq.(\ref{kl_h})}, the algorithm tries to do two things: 1) align $f$ with $\mu^m_p-\mu^m_q$ and 2) increase norms $||\mu^m_p-\mu^m_q||_{\mathcal{H}_K}$ and $||f||_{\mathcal{H}_K}$. For fixed $\langle\mu^m_p,\mu^m_q\rangle_{\mathcal{H}_K}/(||\mu^m_p||.||\mu^m_q||)$, we can show that the ratio $||(\mu^m_p-\mu^m_q)||_{\mathcal{H}_K}/||\mu^m_p+\mu^m_q||_{\mathcal{H}_K}$ also remains unchanged. Under this assumption, we could say that maximizing eq.(\ref{kl_h}) could lead to increment of ${||\mu^m_p+\mu^m_q||}_{\mathcal{H}_K}$, and nothing would stop the network from going this way. When this happens, the inequality in eq.(\ref{sup_avg}) suggests that $S_K(\theta)$ also increases, thereby {increasing the} probability of deviation-from-the-mean error {in the KL estimate} by Theorem \ref{complexity}. In other words, as we train the neural discriminator, the neural network parameters $\theta$ change such that the complexity of RKHS might itself keep increasing which causes exponential growth in the sample complexity.

\begin{figure}[tbp]
\begin{center}
\centerline{\includegraphics[width=\linewidth]{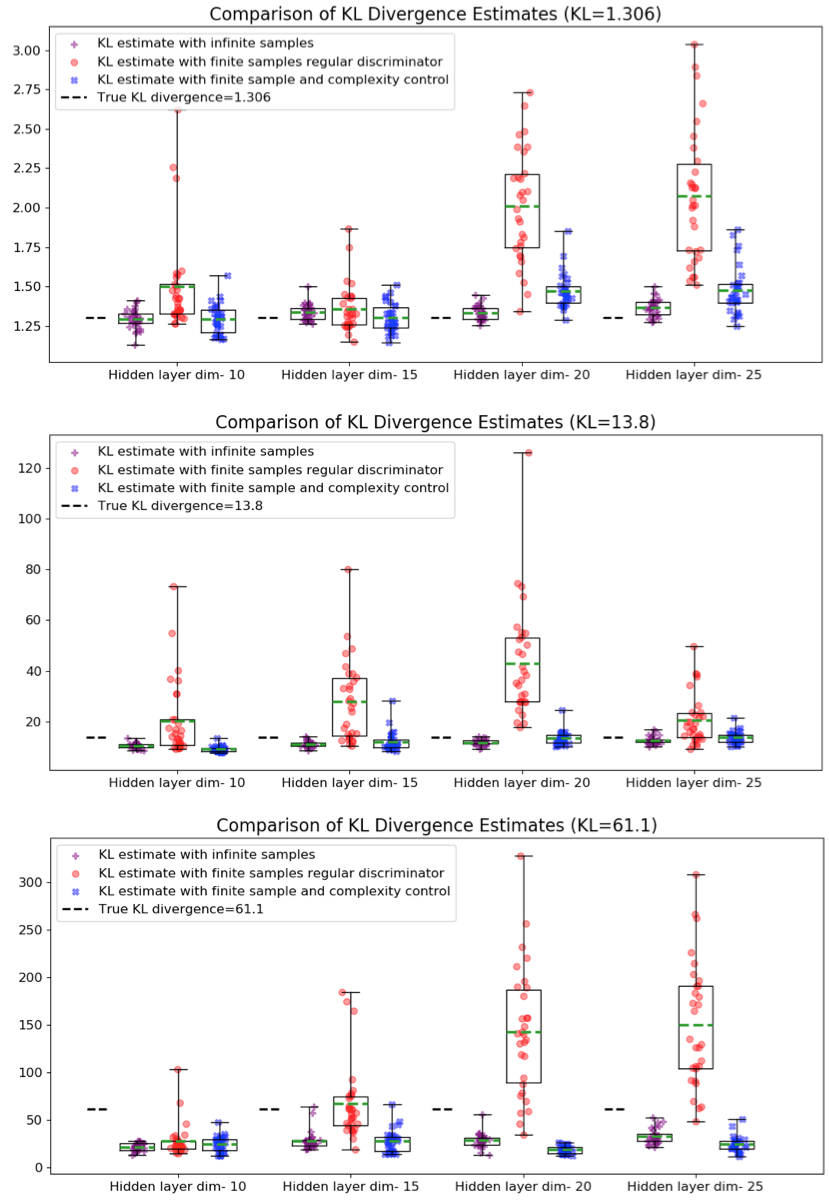}}
\caption{Comparison of KL divergence estimates using 
infinite samples (purple) versus finite samples with (blue) and without (red) complexity control. 
\vspace{-.2cm}}
\label{three_kl}
\end{center}
\vspace{-.5cm}
\end{figure}
To control the complexity of RKHS space, we can control $S_K(\theta)$ 
during the training of the neural network.
To do this in a scalable way compatible with neural networks, we use gradient descent based optimization. Computation of gradient of $S_K(\theta)$ w.r.t $\theta$ is straightforward using definition of $K_{\theta}(x,t)$ and can be easily realized by using backpropagation.  Ideally, $S_K(\theta)$ is $max.K_{\theta}(x,t)$ over all the data-pairs $(x,t) \in \mathcal{X}\times \mathcal{X}$, which requires passing all the datapoints through neural network.  
Instead, we simply compute supremum over the minibatch matrix which contains the $2b\times 2b$ entries corresponding to every pair in $2b$ elements ($b$ from each distribution $p(x)$ and $q(x)$). This is obviously a lowerbound -- denoted by $S_{mini}(\theta)$ -- of $S_K(\theta)$. To penalize the RKHS space that are high in complexity, we add a regularization term with parameter $\lambda$ to maximize a modified objective:
\begin{align}
\label{aug_obj}
   {\frac{1}{m}\sum_{x_i \sim p(x_i)}\log \sigma(f(x_i))+\frac{1}{m}\sum_{x_j \sim q(x_j)}\log (1-\sigma(f(x_j)))} - \lambda.S_{mini}^{\gamma}
\end{align}
where $\gamma$ is an estimation of $\frac{n}{h}$ and 
treated
as a hyperparameter. 
Optimization of eq.(\ref{aug_obj}) w.r.t. neural network parameters $\theta$ allows dynamic control of the complexity of the discriminator function on the fly in a scalable and efficient way.

\section{Experimental Results}
\textbf{Experimental Setup:}
We assume that we have finite sets of samples from two distributions. We further assume that we are required to apply minibatch based optimization.  We consider estimating KL divergence in a simple case of two Gaussian distributions in 2D, where we know the {analytical} KL divergence between the two distributions as the ground truth. We consider three different pairs of distributions corresponding to true KL divergence values of $1.3, 13.8$ and $61.1$, respectively and use $m=5000$ samples from each distribution to estimate KL in the finite case.

As the discriminator, we use a fully connected neural network with two hidden layers. The number of hidden units are varied to understand the effect of the discriminator complexity {on the fluctuation of the KL estimate}. The dimensions are kept identical between the neural-net discriminator and the RKHS discriminator, the latter being different only in that its last layer is stochastic. We perform random estimation experiment 30 times and report the mean, standard deviation, scatter and box plots

\textbf{Finite \textit{v.s.} Infinite Samples:}
In infinite samples experiment, we assume 
that we can continuously sample 
from the model generating data 
from the two given distributions. 
The results of KL estimates using infinite samples is shown in Fig.~\ref{three_kl} and Table.\ref{kl_table} left, 
in comparison with estimates using finite samples without controlling the complexity of the neural-net discriminator.
We observe that when we use infinite samples, we obtain an estimate with low variance and values close to the analytical truth in KL = $1.3$ and KL = $13.8$ and an underestimate when KL = $61.1$. In contrast, when we use finite samples without controlling the complexity of the neural-net discriminator, the estimates fluctuated heavily confirming our hypothesis:  
we need to control the complexity of the function when the number of samples is finite, or else the probability of estimation error increases.

\textbf{Complexity Control:}
Fig.~\ref{three_kl} and Table.\ref{kl_table} left compare the estimation of KL divergence with and without controlling the discriminator complexity. With discriminator complexity penalized in eq. (\ref{aug_obj}), the KL estimates are much more reliable (low variance) and 
closer to the estimates from infinite samples.
Note in Fig.\ref{three_kl} that, without complexity penalization, the erratic behavior of the KL estimator  worsens as the number of hidden layers increases in the  discriminator. This is consistent with our theory because increasing the number of hidden layers increases the complexity of the discriminator.
{This highlights the need of higher degree penalization of the discriminator complexity as a neural network with increased capacity is used to estimate higher values of KL divergence.}

\begin{table}[tb]
\caption{Left: Comparison of KL divergence estimates using different methods (hidden layer dim =25). Right: The effect of the regularization parameter $\boldsymbol{\lambda}$; hidden layer dimension = 20}
\begin{minipage}[c]{0.54\textwidth}
\centering
\begin{tabular}{ p{3em} p{1.5cm} p{1.5cm} p{1.35cm}}
\hline \hline
\multicolumn{1}{c}{\multirow{2}{4em}{\bf Method}}  &\multicolumn{3}{c}{\bf True KL} \\
\cline{2-4}
&1.3& 13.8 & 61.1\\
\hline 
MINE   & Unstable & Unstable & Unstable \\
VDM   & Unstable & Unstable & Unstable\\
\hline
Infinite sample & $1.4\pm 0.05$ & $12.6\pm 1.5$ & $32.4 \pm 7.9$\\
\hline
NN Disc & $2.1\pm 0.42$ & $20.6\pm 9.8$ & $150 \pm 65$\\
\hline
Complexity control & $1.5\pm 0.15$ & $13.6\pm 2.4$ & $24 \pm 8.2$\\
\hline 
\end{tabular}
\end{minipage}
\hspace{0.05cm}
\begin{minipage}[c]{0.45\textwidth}
\begin{center}
\begin{tabular}{ p{3em} p{1cm} p{1cm} p{1cm}}
\hline \hline
\multicolumn{1}{c}{\multirow{2}{4em}{$\boldsymbol{\lambda}$}}  &\multicolumn{3}{c}{\bf True KL} \\
\cline{2-4}
&1.3& 13.8 & 61.1\\
\hline 

5e-5 & $1.46\pm 0.22$ & $16.65\pm 10.4$ & $116.7 \pm 116$\\
\hline
1e-4 & $1.56\pm 0.25$ & $30.97\pm 10.5$ & $39.17 \pm 18.5$\\
\hline
5e-4 & $1.47\pm 0.11$ & $13.44\pm 2.68$ & $18.36 \pm 3.9$\\
\hline 
\end{tabular}
\end{center}
\end{minipage}
\label{kl_table}
\vspace{-.2cm}
\end{table}

\textbf{Effect of {Regularization} Parameter:} 
Table \ref{kl_table} right shows the effect of the regularization parameter $\lambda$ that 
tunes the level of complexity control in 
eq. (\ref{aug_obj}).
The fluctuation in estimates decreases as we increase the value of $\lambda$. We also repeated these experiments varying the latent dimension (see Fig.3 in appendix) and found that the pattern holds in all cases.
Furthermore, for a discriminator with low complexity (\textit{e.g.}, latent dimension = 10), a smaller value of $\lambda$ is sufficient to yield low-variance estimate. As the size of the hidden layer increases,
we need to penalize the complexity aggressively with a higher value of $\lambda$ in order to obtain the same level of consistency. This further supports our theory.

\textbf{Underestimation for High KL Divergence}
We observe in Fig.\ref{three_kl} and Table.~\ref{kl_table} that, for KL = $61.1$, results from both infinite samples and finite samples with complexity control give  underestimated KL divergences even though they reduce fluctuation significantly. This is not surprising since we were focusing on deviation-from-mean error. 
The total estimation error consists of two additional errors: discriminator induced error and the bias (see eq.\ref{total_error}). For the small KL divergence, simply controlling complexity was sufficient to minimize all the errors, but for higher value, it is no longer sufficient. The underestimation might be either because of the bias or error induced by incorrect discriminator function. High bias might be caused if we control the function space too much such that the optimum discriminator $f^*_{\mathcal{H}_K}$ is not close to true discriminator function, $f^*$ (see bias-variance trade-off \cite{cucker2002mathematical, bishop2006pattern}). It would be an interesting future direction to quantify all three error terms in eq.(\ref{total_error}).

\section{Conclusions \& Discussion}
We have shown that using a regular neural network as a discriminator in estimating KL divergence results in unreliable estimation if the complexity of the function space is not controlled.  
We then showed 
a solution by penalizing the kernel complexity in a scalable way using neural networks.

The idea of constructing a {neural-net} function in RKHS and complexity control could also be useful in 
stabilizing GANs, 
or potentially in improving generalization of neural networks. 
Several papers have identified issues with the stability of GANs \cite{mescheder2018gan,kodali2017convergence,thanh-tung2018improving}. One common understanding is that, in its raw form, we do not enforce the discriminator function to be smooth or regular around the neighborhood of its inputs. Currently, the most successful way to stabilize GANs is to enforce smoothness by gradient penalization. Even in variations like Wasserstein GAN \cite{arjovsky2017wasserstein,improved_wgan} and MMD GAN \cite{binkowski2018demystifying}, gradient penalty is crucial to achieve stable results. On the light of the present analysis, we believe that the gradient penalty can be thought as one way to control the complexity of the discriminator. The objective and nature of optimization is such that the complexity of discriminator is bound to increase and therefore some way of decreasing complexity is a must. Similarly, generalization of neural network classifiers and regressors could be improved with complexity control.
\bibliographystyle{plain}
\bibliography{bibli}

\begin{thebibliography}{10}

\bibitem{arjovsky2017wasserstein}
Martin Arjovsky, Soumith Chintala, and L{\'e}on Bottou.
\newblock Wasserstein generative adversarial networks.
\newblock In {\em International Conference on Machine Learning}, pages
  214--223, 2017.

\bibitem{bach2017breaking}
Francis Bach.
\newblock Breaking the curse of dimensionality with convex neural networks.
\newblock {\em The Journal of Machine Learning Research}, 18(1):629--681, 2017.

\bibitem{bach2017equivalence}
Francis Bach.
\newblock On the equivalence between kernel quadrature rules and random feature
  expansions.
\newblock {\em The Journal of Machine Learning Research}, 18(1):714--751, 2017.

\bibitem{belghazi2018mutual}
Mohamed~Ishmael Belghazi, Aristide Baratin, Sai Rajeshwar, Sherjil Ozair,
  Yoshua Bengio, Aaron Courville, and Devon Hjelm.
\newblock Mutual information neural estimation.
\newblock In {\em International Conference on Machine Learning}, pages
  531--540, 2018.

\bibitem{berlinet2011reproducing}
A.~Berlinet and C.~Thomas-Agnan.
\newblock {\em Reproducing Kernel Hilbert Spaces in Probability and
  Statistics}.
\newblock Springer US, 2011.

\bibitem{bishop2006pattern}
Christopher~M Bishop.
\newblock {\em Pattern recognition and machine learning}.
\newblock springer, 2006.

\bibitem{binkowski2018demystifying}
Mikołaj Bińkowski, Dougal~J. Sutherland, Michael Arbel, and Arthur Gretton.
\newblock Demystifying {MMD} {GAN}s.
\newblock In {\em International Conference on Learning Representations}, 2018.

\bibitem{chen2018isolating}
Tian~Qi Chen, Xuechen Li, Roger~B Grosse, and David~K Duvenaud.
\newblock Isolating sources of disentanglement in variational autoencoders.
\newblock In {\em Advances in Neural Information Processing Systems}, pages
  2610--2620, 2018.

\bibitem{chen2016infogan}
Xi~Chen, Yan Duan, Rein Houthooft, John Schulman, Ilya Sutskever, and Pieter
  Abbeel.
\newblock Infogan: Interpretable representation learning by information
  maximizing generative adversarial nets.
\newblock In {\em Advances in neural information processing systems}, pages
  2172--2180, 2016.

\bibitem{cucker2002mathematical}
Felipe Cucker and Steve Smale.
\newblock On the mathematical foundations of learning.
\newblock {\em Bulletin of the American mathematical society}, 39(1):1--49,
  2002.

\bibitem{gretton2012kernel}
Arthur Gretton, Karsten~M Borgwardt, Malte~J Rasch, Bernhard Sch{\"o}lkopf, and
  Alexander Smola.
\newblock A kernel two-sample test.
\newblock {\em Journal of Machine Learning Research}, 13(Mar):723--773, 2012.

\bibitem{improved_wgan}
Ishaan Gulrajani, Faruk Ahmed, Martin Arjovsky, Vincent Dumoulin, and Aaron~C
  Courville.
\newblock Improved training of wasserstein gans.
\newblock In I.~Guyon, U.~V. Luxburg, S.~Bengio, H.~Wallach, R.~Fergus,
  S.~Vishwanathan, and R.~Garnett, editors, {\em Advances in Neural Information
  Processing Systems 30}, pages 5767--5777. Curran Associates, Inc., 2017.

\bibitem{kodali2017convergence}
Naveen Kodali, Jacob Abernethy, James Hays, and Zsolt Kira.
\newblock On convergence and stability of gans.
\newblock {\em arXiv preprint arXiv:1705.07215}, 2017.

\bibitem{mescheder2018gan}
Lars Mescheder, Andreas Geiger, and Sebastian Nowozin.
\newblock Which training methods for gans do actually converge?
\newblock In {\em International Conference on Machine Learning}, pages
  3481--3490, 2018.

\bibitem{Mescheder2017ICML}
Lars Mescheder, Sebastian Nowozin, and Andreas Geiger.
\newblock Adversarial variational bayes: Unifying variational autoencoders and
  generative adversarial networks.
\newblock In {\em International Conference on Machine Learning (ICML)}, 2017.

\bibitem{nguyen2010estimating}
XuanLong Nguyen, Martin~J Wainwright, and Michael~I Jordan.
\newblock Estimating divergence functionals and the likelihood ratio by convex
  risk minimization.
\newblock {\em IEEE Transactions on Information Theory}, 56(11):5847--5861,
  2010.

\bibitem{nowozin2016f}
Sebastian Nowozin, Botond Cseke, and Ryota Tomioka.
\newblock f-gan: Training generative neural samplers using variational
  divergence minimization.
\newblock In {\em Advances in neural information processing systems}, pages
  271--279, 2016.

\bibitem{sonderby2016amortised}
Casper~Kaae S{\o}nderby, Jose Caballero, Lucas Theis, Wenzhe Shi, and Ferenc
  Husz{\'a}r.
\newblock Amortised map inference for image super-resolution.
\newblock {\em ICLR}, 2017.

\bibitem{song2019understanding}
Jiaming Song and Stefano Ermon.
\newblock Understanding the limitations of variational mutual information
  estimators.
\newblock {\em arXiv preprint arXiv:1910.06222}, 2019.

\bibitem{sriperumbudur2010hilbert}
Bharath~K Sriperumbudur, Arthur Gretton, Kenji Fukumizu, Bernhard
  Sch{\"o}lkopf, and Gert~RG Lanckriet.
\newblock Hilbert space embeddings and metrics on probability measures.
\newblock {\em Journal of Machine Learning Research}, 11(Apr):1517--1561, 2010.

\bibitem{thanh-tung2018improving}
Hoang Thanh-Tung, Truyen Tran, and Svetha Venkatesh.
\newblock Improving generalization and stability of generative adversarial
  networks.
\newblock In {\em International Conference on Learning Representations}, 2019.

\end{thebibliography}

\newpage
\appendix
\section{Bounding the Error Probability of KL Estimates}
To obtain the probability of deviation-from-mean error, we first bound the error probability in terms of the covering number in Lemma \ref{sample_complexity}. Then, we use an estimate of the covering number of RKHS due to \cite{cucker2002mathematical} to obtain a bound of error probability in terms of the kernel $K_{\theta}$ in Theorem \ref{complexity}.
\begin{customlemma}{1}
Let $f^m_{\mathcal{H}_K}$ be the optimal discriminator function in a RKHS $\mathcal{H}_{K}$ which is M-bounded. Let ${KL}_m(f^m_{\mathcal{H}_K})=\frac{1}{m}\sum_i f^m_{\mathcal{H}_K}(x_i)$ and $KL(f^m_{\mathcal{H}_K}) = E_{ p(x)}[f^m_{\mathcal{H}_K}(x)]$ be the estimate of KL divergence from m samples and that by using true distribution $p(x)$ respectively.
Then the probability of error at some accuracy level, $\epsilon$ is lower-bounded as:
\begin{align}
\nonumber
    \text{Prob.}(&|{KL}_m(f^m_{\mathcal{H}_K})-{KL}(f^m_{\mathcal{H}_K})|\leq \epsilon) 
    \geq 1-2\mathcal{N}(\mathcal{H}_K, \frac{\epsilon}{4\sqrt{S_K}})\exp(-\frac{m\epsilon^2}{4M^2})
\end{align}
where $\mathcal{N}(\mathcal{H}_K,\eta)$ denotes the covering number of a RKHS space $\mathcal{H}_K$ with disks of radius $\eta$, and $S_K=\underset{x,t}{sup} $\hspace{0.1cm} ${K(x,t)}$ which we refer as kernel complexity
\end{customlemma}
\begin{proof}
Let $\ell_z(f)=E_{p(x)}[f(x)]-\frac{1}{m}\sum_i f(x_i)$ denotes the error in the estimate such that we want to bound $|\ell_z(f)|$. We have,
\begin{align*}
    &\ell_z(f_1)-\ell_z(f_2)
    = E_{p(x)}[f_1(x)- f_2(x)]-\frac{1}{m}\sum_i f_1(x_i) -  f_2(x_i)
\end{align*}
We know $E_{p(x)}[f_1(x)- f_2(x)]\leq ||f_1-f_2||_\infty$ and $\frac{1}{m}\sum_i f_1(x_i) -  f_2(x_i) \leq ||f_1-f_2||_\infty$.
Using the triangle inequality, we obtain
$
| \ell_z(f_1)-\ell_z(f_2) | \leq 2||f_1-f_2||_\infty
$. Now, consider $f\in \mathcal{H}_K$, then,
\begin{align}
    |f(x)|=|\langle K_x, f \rangle| \leq ||f||||K_x||=||f||\sqrt{K(x,x)}
\end{align}
This implies the RKHS space norm and $\ell_\infty$ norm of a function are related by 
\begin{align}
\label{sup_rkhs}
    ||f||_\infty \leq \sqrt{S_K}||f||_{\mathcal{H}_K}
\end{align}
Hence, we have: 
\begin{align}
\label{lipschitz}
    | \ell_z(f_1)-\ell_z(f_2) | \leq 2\sqrt{S_K}||f_1-f_2||_{\mathcal{H}_K}
\end{align}
The idea of the covering number is to cover the whole RKHS space $\mathcal{H}_K$ with disks of some fixed radius $\eta$, which helps us bound the error probability in terms of the number of such disks. Let $\mathcal{N}(\mathcal{H}_K,\eta)$ be such disks covering the whole RKHS space. Then, for any function $f$ in $\mathcal{H}_K$, we can find some disk, $D_j$ with centre $f_j$, such that $|| f-f_j ||_{\mathcal{H}_K} \leq \eta$. If we choose $\eta= \frac{\epsilon}{2\sqrt{S_K}}$, then from eq.(\ref{lipschitz}), we obtain,
\begin{align}
    \label{2eps}
    \underset{f\in D_j}{sup}{| \ell_z(f)| \geq 2\epsilon} \implies  | \ell_z(f_j)| \geq \epsilon
\end{align}
Using the Hoeffding's inequality,\hspace{0.1cm}
$
    \text{Prob.}(|\ell_z(f_j)|\geq \epsilon )\leq 2e^{-\frac{m\epsilon^2}{2M^2}}
$
and eq.(\ref{2eps}),
\begin{align}
    &\text{Prob.}(\underset{f\in D_j}{sup}{| \ell_z(f)| \geq 2\epsilon} )\leq 2e^{-\frac{m\epsilon^2}{2M^2}}
\end{align}

Applying union bound over all the disks, we obtian,
\begin{align}
    &\text{Prob.}(\underset{f\in \mathcal{H}}{sup}{| \ell_z(f)| \geq 2\epsilon} )\leq 2\mathcal{N}(\mathcal{H},\frac{\epsilon}{2\sqrt{S_K}})e^{-\frac{m\epsilon^2}{2M^2}}\\
    \nonumber
    &\text{Prob.}(\underset{f\in \mathcal{H}}{sup}{| \ell_z(f)| \leq \epsilon} )\geq 1- 2\mathcal{N}(\mathcal{H},\frac{\epsilon}{4\sqrt{S_K}})e^{-\frac{m\epsilon^2}{8M^2}}
\end{align}
which proves the lemma. \\

\underline{On M-boundedness of $f^m_{\mathcal{H}_K}$}\\
To prove the lemma, we assumed that $f^m_{\mathcal{H}_K}$ is M bounded. To see why this is reasonable, from eq.\ref{sup_rkhs}, we have $||f^m_{\mathcal{H}_K}||_\infty \leq \sqrt{S_K}||f^m_{\mathcal{H}_K}||_{\mathcal{H}_K}$. Since by construction, $||f^m_{\mathcal{H}_K}||_{\mathcal{H}_K} \leq 1$, $f^m_{\mathcal{H}_K}$ is bounded if $S_K$ is bounded, which is true by assumption and seems to hold true in experiments.
\end{proof}
\begin{remark}
We derived the error bound based on the Hoeffding's inequality by assuming that our only knowledge about $f$ is that it is bounded. If we have other knowledge, for example, if we know the variance of $f$, we could use Bernstein's inequality instead of Hoeffding's inequality with minimal change to the proof. To the extent we are interested in the contribution of neural network in error bound, however, there is not much gain by using one inequality or the other. Hence, we stick with Hoeffding's inequality and note other possibilities.
\end{remark} 
\begin{remark}
Note that in Lemma 1, the radius of disks are inversely related to the the quantity, $S_K$, meaning that if $S_K$ is high, we would need large number of disks to fill the RKHS space. Hence, it denotes a quantity that reflects the complexity of the RKHS space. We, therefore, term it kernel complexity. Also in eq. \ref{sup_rkhs} and the discussion about the M-boundedness, we see that the maximum value $|f(x)|$ depends on $S_K$, again providing insight into how $S_K$ may control both maximum fluctuation and the boundedness.

\end{remark}
Lemma \ref{sample_complexity} bounds the probability of error in terms of the covering number of the RKHS space. Next, we use Lemma 2 due to \cite{cucker2002mathematical} 
to obtain an error bound in estimating KL divergence with finite samples in Theorem \ref{complexity}.

\begin{customlemma}{2}[\cite{cucker2002mathematical}]
Let $K: \mathcal{X}\times \mathcal{X}\to {\rm I\!R} $ is a $\mathcal{C}^\infty$ Mercer kernel and the inclusion $I_K:\mathcal{H}_K\xhookrightarrow{}\mathcal{C}(\mathcal{X})$ is the compact embedding defined by $K$ to the Banach space $\mathcal{C}(\mathcal{X})$ . Let $B_R$ be the ball of radius $R$ in RKHS $\mathcal{H}_{K}$. Then $\forall \eta>0, R >0, h>n $, we have
\begin{align}
    \ln \mathcal{N}(I_K(B_R), \eta) \leq \left( \frac{RC_h}{\eta} \right)^{\frac{2n}{h}}
\end{align}
where $\mathcal{N}$ gives the covering number of the space $I_K(B_R)$ with disks of radius $\eta$, and $n$ represents the dimension of inputs space $\mathcal{X}$. $C_h$ is given by 
\begin{align}
C_h=C_s\sqrt{||L_s||}
\end{align}
where $L_s$ is a linear embedding from square integrable space $\mathcal{L}_2(d\rho)$ to the Sobolev space $H^{h/2}$ and $C_s$ is a constant.
\end{customlemma}
To prove Lemma \ref{covering number}, the RKHS space is embedded in the Sobolev Space $H^{h/2}$ using $L_s$ and then covering number of Sobolev space is used. Thus the norm of $L_s$ and the degree of Sobolev space, $h/2$, appears in the covering number of a ball in $\mathcal{H}_K$. 
In Theorem \ref{complexity}, we use this Lemma to bound the {estimation error of KL divergence}. 

\begin{customthm}{2}
Let ${KL}(f^m_{\mathcal{H}_K})$ and ${KL}_m(f^m_{\mathcal{H}_K})$ be the estimates of KL divergence obtained by using true distribution $p(x)$ and $m$ samples respectively as described in Lemma \ref{sample_complexity}, then the probability of error in the estimation at the error level $\epsilon$ is given by:
\begin{align*}
   \text{Prob.}(&|{KL}_m(f^m_{\mathcal{H}_K})-{KL}(f^m_{\mathcal{H}_K})|\leq \epsilon) \geq 1-2\exp\Bigg[\left( \frac{4RC_s\sqrt{S_K||L_s||}}{\epsilon} \right)^{\frac{2n}{h}}-\frac{m\epsilon^2}{4M^2}\Bigg]
\end{align*}
\end{customthm}
\begin{proof}
Lemma \ref{covering number} gives the covering number of a ball of radius $R$ in a RKHS space. If we consider the hypothesis space to be a ball of radius $R$ in Lemma \ref{sample_complexity} , we can apply Lemma \ref{covering number} in it. Additionally, since we fix the radius of disks to be $\eta=\frac{\epsilon}{4\sqrt{S_K}}$ in Lemma \ref{sample_complexity}, we obtain,
\begin{align}
\nonumber
    &\text{Prob.}(|{KL}_m(f^m_{\mathcal{H}_K})-{KL}(f^m_{\mathcal{H}_K})|\leq \epsilon) \geq 1-2\exp\Big[\left( \frac{4\sqrt{S_K}RC_h}{\epsilon} \right)^{\frac{2n}{h}}-\frac{m\epsilon^2}{4M^2}\Big]\\
\end{align}
Substituting $C_h=C_s\sqrt{||L_s||}$ gives the required result.
\end{proof}

\section{Mean Embedding Upper Bound}

Using mean embedding representation of functions in RKHS space, we derive some geometrical insights into the maximization objective.
Theorem \ref{mebub} first gives an upper bound to the optimization objective, which depends on Lemma 3 as proved next.


\begin{customlemma}{3}
$
\nonumber
{E_{p(x)}\log \sigma(f(x))+E_{q(y)}\log (1-\sigma(f(y)))}
\leq \log \sigma[E_{p(x)}(f(x))-E_{q(y)}(f(y))]
$
\end{customlemma}
\begin{proof}
\begin{align*}
{E_{p(x)}\log \sigma(f(x))+E_{q(x)}\log (1-\sigma(f(x)))}
&={E_{p(x)}\log \sigma(f(x))+E_{q(x)}\log (\frac{1}{1+\exp(f)})}\\
&={E_{p(x)}\log \sigma(f(x))+E_{q(x)}\log (\frac{\exp(-f)}{1+\exp(-f)})}\\
&={E_{p(x)}\log \sigma(f(x))+E_{q(x)}\log \sigma(-f(x))}\\
&\leq\log\sigma[{E_{p(x)} (f(x))]+\log \sigma [E_{q(x)}(-f(x))}]\\
&\leq\log\sigma[{E_{p(x)} (f(x))-E_{q(x)}(f(x))}
\end{align*}
where we used the fact that $\log \sigma $ is a concave function and applied Jensen's inequality in last two lines and linearity of expectation in the last line.
\end{proof}


\section{Algorithm}
Algorithm 1 details the algorithm to estimate KL divergence with complexity control.
\begin{algorithm}[h]
\caption{KL divergence estimation with complexity control}
\begin{algorithmic}[1]
\STATE Fix minibatch size, $b$, hyperparameter $\gamma$, number of samples $m$, $flat\_n=100$, $idx=0$,$\ell_{min}=\infty$ 
\STATE Initialize the neural network parameters ${\theta}$, last layer $w \sim \mathcal{N}(\bar{w},LL^T)$, such that $\bar{w}=0,LL^T=\text{I}$
\FOR{iteration $iter$ in 1 to $iter_{max}$}
    \STATE $kl_{sum}=0, \ell_{adv}=0$, 
    $n\_batch=(m/b)$
    \FOR{iteration k in 1 to $n\_batch$}
        \STATE Sample minibatch $\{x_i\}_{i=1}^b$ from $p(x)$ and $\{y_i\}_{i=1}^b$ from $q(x)$, and $J=\{\{x_i\}_{i=1}^b, \{y_i\}_{i=1}^b\}$
        \STATE For each $x_i, y_i$, sample $\epsilon\sim \mathcal{N}(0,I)$ and obtain samples $\{w_j\}_{j=1}^d$ where $w_j=\bar{w}+L\epsilon_j$  
        \STATE $f(x)_i, =\frac{1}{d}\sum_j\phi_{\theta}(x_i)^Tw_j$\\
        $f(y)_i=\frac{1}{d}\sum_j\phi_{\theta}(y_i)^Tw_j$
        \STATE $loss_{d}=-\frac{1}{b}\sum_i\log\sigma(f(x)_i)+\log(1-f(y)_i)$
        \STATE $S_{mini}=\underset{x \in J, t \in J}{max} \phi_{\theta}(x)^T(\bar{w}\bar{w}^T+\Sigma)\phi_{\theta}(t) $
        \STATE Backpropagate $loss=loss_{d}+\lambda.S_{mini}^{\gamma}$ and update $\theta,\bar{w},L$
        \STATE $kl_{sum}=kl_{sum}+\frac{1}{b}\sum_i\log\sigma(f(x)_i)$\\
        $\ell_{adv}=\ell_{adv}+loss_{d}$
        
    \ENDFOR
    \STATE $\ell=\ell_{adv}/n\_{batch}$, $kl_{iter}=kl_{sum}/n\_batch$
    \IF{$\ell <\ell_{min}$}
        \STATE $kl=kl_{iter}$, $idx=iter$
    \ELSIF{$iter >idx+flat\_n$}
        \RETURN $kl$
    \ENDIF
\ENDFOR
\RETURN $kl$
\end{algorithmic}
\end{algorithm}

\section{Experimental Results}



\subsection{Details of experimental setup}

\textbf{Neural RKHS discriminator architecture} (Proposed method)\\
\texttt{
Fully connected \\
Leaky ReLU\\
Fully connected\\
Leaky ReLU
}

For the RKHS discriminator, this gives $\phi_{\theta}(x)$ for the input data $x$. Then, $f(x)$ needs to defined as in line 8 of Algorithm 1. Similarly, total loss with complexity penalization is computed as line 11 in Algorithm 1.

\textbf{Neural network discriminator architecture}\\
\texttt{
Fully connected \\
Leaky ReLU\\
Fully connected\\
Leaky ReLU\\
Fully connected
}

For the Neural net discriminator, this would directly give $f(x)$ for the input data $x$. Also, loss would be defined by line 9, no penalization as in line 11 of Algorithm 1.

\textbf{Learning rate:} $5\times 10^{-3}$\\
$\boldsymbol{\gamma:} 0.05$ \\
\textbf{No. of samples from each distribution:} $5000$\\
\textbf{Minibatch size:} 50\\
\textbf{Hyperparameter selection:} The hyperparameters like learning rate and $\gamma$
 were selected by first estimating KL divergence at a mid value like $13$. Then, same value was used in all experiments.



\begin{figure}[t]
\begin{center}
\centerline{\includegraphics[width=\linewidth]{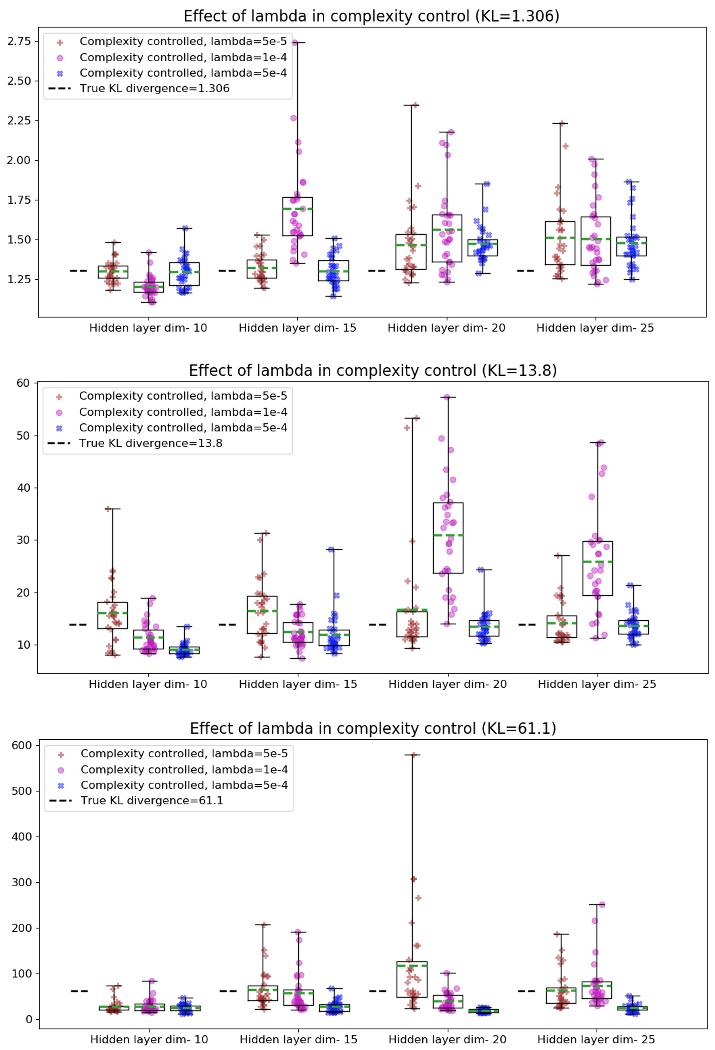}}
\caption{The effect of the regularization parameter $\lambda$ in KL estimates (y-axis) plotted against the varying hidden layer dimension for each KL divergence value.}
\label{betakl}
\end{center}
\vspace{-6mm}
\end{figure}


\subsection{Effect of {Regularization} Parameter $\lambda$}
Fig.~\ref{betakl} shows the effect of the regularization parameter $\lambda$ that 
tunes the level of complexity control in estimating the KL divergence.
As expected, in all cases, we observe that the fluctuation in estimates decreases as we increase the value of $\lambda$.
\end{document}